\documentclass[conference]{IEEEtran}

\usepackage[utf8]{inputenc}
\usepackage{amssymb,amsmath,bbm,mathrsfs,amsthm,bm,enumitem,csquotes,mathtools,bookmark,multicol,pgf,tikz}
\usetikzlibrary{arrows,shapes,automata,positioning,calc}

\newtheorem{theorem}{Theorem}[section]
\newtheorem{corollary}[theorem]{Corollary}
\newtheorem{lemma}[theorem]{Lemma}

\providecommand{\N}{{\ensuremath{\mathbb{N}}}}
\providecommand{\R}{{\ensuremath{\mathbb{R}}}}
\providecommand{\1}{{\ensuremath{\mathbbm{1}}}}
\def\r#1{{\ensuremath{\mathcal{R}\hspace{-0.1em}{#1}}}}
\newcommand{\eps}{\varepsilon}
\renewcommand{\ae}{a.e.\@ }
\DeclareMathOperator{\ReLU}{ReLU}
\newcommand{\nder}{{\ensuremath{\mathcal{D}}}}
\newcommand{\der}{{\ensuremath{D}}}

\begin{document}
\title{Towards a regularity theory for {ReLU} networks -- chain rule and global error estimates}
\author{
\IEEEauthorblockN{Julius Berner\IEEEauthorrefmark{1}, Dennis Elbr\"achter\IEEEauthorrefmark{1}, Philipp Grohs\IEEEauthorrefmark{3}, Arnulf Jentzen\IEEEauthorrefmark{4}}
\IEEEauthorblockA{\IEEEauthorrefmark{1}Faculty of Mathematics, University of Vienna\\Oskar-Morgenstern-Platz 1, 1090 Vienna, Austria\\}
\IEEEauthorblockA{\IEEEauthorrefmark{3}Faculty of Mathematics and Research Platform DataScience@UniVienna,
University of Vienna \\ Oskar-Morgenstern-Platz 1, 1090 Vienna, Austria\\}
\IEEEauthorblockA{\IEEEauthorrefmark{4}Department of Mathematics, ETH Z\"urich \\ R\"amistrasse 101, 8092 Z\"urich, Switzerland\\}
}

\maketitle
\begin{abstract}
Although for neural networks with locally Lipschitz continuous activation functions the classical derivative exists almost everywhere, the standard chain rule is in general not applicable.
We will consider a way of introducing a derivative for neural networks that admits a chain rule, which is both rigorous and easy to work with. In addition we will present a method of converting approximation results on bounded domains to global (pointwise) estimates. 
This can be used to extend known neural network approximation theory to include the study of regularity properties. 
Of particular interest is the application to neural networks with ReLU activation function, 
where it contributes to the understanding of the success 
of deep learning methods for high-dimensional partial differential equations.
\end{abstract}

\IEEEpeerreviewmaketitle
\section{Introduction}
It has been observed that deep neural networks exhibit the remarkable capability of overcoming the curse of dimensionality in a number of different scenarios. In particular, for certain types of high-dimensional partial differential equations (PDEs) there are promising empirical observations~\cite{beckbecker2018,E2017,2017Jentzen,SirignanoSpiliopoulos2017,FujiiTakahashiTakahashi2017,KhooLuYing2017,EYu2017} backed by theoretical results for both the approximation error~\cite{Elbrachter2018DNNPricing,grohs2018approx,Jentzen2018ACoefficients,Hutzen2019} as well as the generalization error~\cite{Berner2018AnalysisEquations}.
In this context it becomes relevant to not only show how well a given function of interest can be approximated by neural networks but also to extend the study to the derivative of this function. 
A number of recent publications~\cite{Yarotsky2018,Petersen2017,Grohs2019DeepTheory} have investigated the required size of a network 
which is sufficient to approximate certain interesting (classes of) functions within a given accuracy. 
This is achieved, first, by considering the approximation of basic functions by very simple networks 
and, subsequently, by combining those networks in order to approximate more difficult structures. 
To extend this approach to include the regularity of the approximation, one requires some kind of chain rule for the composition of neural networks.
For neural networks with differentiable activation function the standard chain rule is sufficient.  
It, however, fails when considering neural networks with an activation function, which is not everywhere differentiable.
Although locally Lipschitz continuous functions are w.r.t the Lebesgue measure almost everywhere (a.e.) 
differentiable, the standard chain rule is not applicable, as, in general, 
it does not hold even in an 'almost everywhere' sense.
We will introduce derivatives of neural networks in a way that admits a chain rule which is both rigorous as well as easy to work with.
Chain rules for functions which are not everywhere differentiable have been considered in a more general setting in e.g. \cite{murat2003chain,ambrosio1990general}. 
We employ the specific structure of neural networks to get stronger results using simpler arguments. In particular it allows for a stability result, i.e. Lemma~\ref{lem:stab}, the application of which will be discussed in Section V.
We would also like to mention a very recent work \cite{guhring2019error} about approximation in Sobolev norms, where they deal with the issue by using a general bound for the Sobolev norm of the composition of functions from the Sobolev space $W^{1,\infty}$. Note however that this approach leads to a certain factor depending on the dimensions of the domains of the functions, which can be avoided with our method.
For ease of exposition, we formulate our results for neural networks with the ReLU activation function. We, however, consider in Section IV how such a chain rule can be obtained for any activation function which is locally Lipschitz continuous (with at most countably many points at which it is not differentiable). 
In Section V we briefly sketch how the results from Section III can be utilized to get approximation results for certain classes of functions. Subsequently, in Section VI, we present a general method of deriving global error estimates from such approximation results, which are naturally obtained for bounded domains.
Ultimately, we discuss how our results can be used to extend known theory, enabling the further study of the approximation of PDE solutions by neural networks.

\section{Setting}
As in~\cite{Petersen2017}, we consider a neural network $\Phi$ to be a finite sequence of matrix-vector pairs, i.e.
\begin{align}\label{eq:NNdef}
    \Phi=((A_k,b_k))_{k=1}^L,
\end{align}
where $A_k\in\R^{N_k\times N_{k-1}}$ and $b_k\in\R^{N_k}$ for some depth $L\in\N$ and layer dimensions $N_0,N_1,\dots,N_L\in\N$. 
The realization of the neural network $\Phi$ is the function $\r\Phi\colon \R^{N_0}\to\R^{N_L}$ given by
\begin{align}\label{eq:Rdef}
    \r\Phi= W_L\circ\ReLU\circ\,W_{L-1}\circ\ldots\circ\ReLU\circ\,W_1,
\end{align}
where 
$W_k(x)=A_k x + b_k$ for every $x\in\R^{N_k}$ and where
\begin{align}
    \ReLU(x):=(\max\{0,x_1\},\dots,\max\{0,x_N\})
\end{align}
for every $x\in\R^N$. We distinguish between a neural network and its realization, since $\Phi$ uniquely induces $\r\Phi$, while in general there can be multiple non-trivially different neural networks with the same realization.
The representation of a neural network as a structured set of weights as in \eqref{eq:NNdef} allows the introduction of notions of network sizes. 
While there are slight differences between various publications, commonly considered quantities are the depth (i.e. number of affine transformations), the connectivity (i.e. number of non-zero entries of the $A_k$ and $b_k$), and the weight bound (i.e. maximum of the absolute values of the entries of the $A_k$ and $b_k$). 
In \cite{Grohs2019DeepTheory} it has been shown that these three quantities determine the length of a bit string which is sufficient to encode the network with a prescribed quantization error.
In the following let 
\begin{align}
\Phi\!=\!((A_k,b_k))_{k=1}^{L}, \quad\Psi\!=\!((\tilde{A}_k,\tilde{b}_k))_{k=1}^{\tilde{L}}   
\end{align}
be neural networks with matching dimensions in the sense that ${\r\Phi\colon\R^d\to\R^m}$ and ${\r\Psi\colon\R^m\to\R^n}$.
We then define their composition as
\begin{align}\begin{split}\label{eq:concDef}
    &\Psi\odot\Phi:=\\
    &\quad\big(((A_k,b_k))_{k=1}^{L-1},(\tilde{A}_1A_{L},\tilde{A}_1 b_{L}+\tilde{b}_1),((\tilde{A}_k,\tilde{b}_k))_{k=2}^{\tilde{L}}\big).
\end{split}\end{align}
Direct computation shows 
\begin{align}\label{eq:compRealizations}
\r(\Psi\odot\Phi)=\r\Psi\circ\r\Phi.
\end{align}
Note that the realization $\r\Phi$ of a neural network $\Phi$ is continuous piecewise linear (CPL) as a composition of CPL functions. Consequently, it is Lipschitz continuous 
and the realization $\r\Phi$ is almost everywhere differentiable by Rademacher's theorem. In particular all three functions in \eqref{eq:compRealizations} are \ae differentiable. This, however, is not sufficient to get the derivative of $\r(\Psi\odot\Phi)$ from the derivatives of $\r\Psi$ and $\r\Phi$ by use of the classical chain rule. Consider the very simple counterexample of $u(x):=\ReLU(x)$ and $v(x):=0$ and formally apply the chain rule, i.e. 
\begin{align}\label{eq:normalCR}
( \der(u\circ v) )(x)=( \der u )(v(x))\cdot ( \der v )(x).
\end{align}
Even though $ (\der u)(y)$ is well-defined for every $y\in\R\backslash\{0\}$, 
the expression $ (\der u )(v(x))$ is defined for no $x\in\R$. In general this problem occurs when the inner function maps a set of positive measure into a set where the derivative of the outer function does not exist. Now in this case, one can directly see that setting $(\der u)(0)$ to any arbitrary value would cause \eqref{eq:normalCR} to provide the correct result since $(\der v)(x)=0$. 

\section{ReLU network derivative}
We proceed by defining the derivative of an arbitrary neural network in a way such that it 
not only coincides \ae with the derivative of the realization, but also admits a chain rule. 
To this end let $H\colon \R^N\to\R^{N\times N}$ be the function given by
\begin{align}\label{eq:Hdef}
    H(x):=\text{diag}(\1_{(0,\infty)}(x_1),\dots,\1_{(0,\infty)}(x_N))
\end{align}
for every $x=(x_1,\dots,x_N)\in\R^N$ 
and let $\mathcal{R}_K\Phi:=\r((A_k,b_k))_{k=1}^{K}$. We then define the neural network derivative of $\Phi$ as the function $\nder\Phi\colon\R^{N_0}\to\R^{N_L\times N_0}$ given by
\begin{align}\label{eq:NNderDef}
    \nder\Phi:=A_L\cdot H(\mathcal{R}_{L-1}\Phi)\cdot A_{L-1}\cdot\ldots\cdot H(\mathcal{R}_1\Phi)\cdot A_1.
\end{align}
Note that this definition is motivated by formally applying the chain rule with the convention that the derivative of $\max\{0,\,\cdot\,\}$ is zero at the origin. 
Now we need to verify that this is justified.
\begin{theorem}\label{thm:main}
It holds for almost every $x\in\R^d$ that 
\begin{align}
    (\nder\Phi)(x)=( \der(\r\Phi) )(x).
\end{align}
\end{theorem}
\begin{proof}
Let 
$v\colon\R^d\to\R^N$ be a locally Lipschitz continuous function, define $w:=\ReLU\circ\,v$, and 
\begin{align}
    L_i:=\{x\in\R^d\colon w_i(x)=0\}=\{x\in\R^d\colon v_i(x)\leq 0\}.
\end{align}
We now use an observation about differentiability on level sets (see e.g. \cite[Thm 3.3(i)]{Evans2015MeasureEdition}), which states that
\begin{align}\label{eq:Evans}
    (\der w_i)(x)=0 \quad\text{for almost every } x\in L_i.
\end{align}
As $w_i(x)=v_i(x)$ for every $x\in\R^d\backslash L_i$, we get \ae
\begin{align}
    \der w_i=\1_{\R^d\backslash L_i}\cdot\der v_i=\1_{(0,\infty)}(v_i)\cdot\der v_i
\end{align}
and consequently
\begin{align}\label{eq:ReLUrule}
    \der(\ReLU\circ\, v)=H(v)\cdot\der v.
\end{align}
The claim follows by induction over the layers $K=1,\dots,L$ of $\Phi$, using \eqref{eq:ReLUrule} with $v=\mathcal{R}_K\Phi$ for the induction step. 
\end{proof}

Note that even for convex $\r\Phi$ the values of $\nder \Phi$ on the nullset do not necessarily lie in the respective subdifferentials of $\r\Phi$, as can be seen in Figure~\ref{fig:square}. Although Theorem~\ref{thm:main} holds regardless of which value is chosen for the derivative of $\max\{0,\,\cdot\,\}$ at the origin, no choice will guarantee that all values of $\nder \Phi$ lie in the respective subdifferentials of $\r\Phi$. 
Here we have set the derivative at the origin to zero, following the convention of software implementations for deep learning applications, e.g. TensorFlow and PyTorch.
Using \eqref{eq:concDef} and \eqref{eq:NNderDef} one can verify by direct computation that $\nder$ obeys the chain rule.

\begin{corollary}\label{thm:chain_rule}
It holds for every $x\in\R^d$ that
\begin{align}\label{eq:chainrule}
    ( \nder(\Psi\odot\Phi) )(x)=( \nder\Psi )( \r\Phi(x) )\cdot (\nder\Phi)(x).
\end{align}
\end{corollary}

Note that \eqref{eq:chainrule} is well-defined as $\nder\Psi$ exists everywhere, although it only coincides with $\der(\r\Psi)$ almost everywhere. Theorem~\ref{thm:main} however guarantees that we still have \ae
\begin{align}
    \nder(\Psi\odot\Phi)=\der(\r(\Psi\odot\Phi))=\der(\r\Psi\circ\r\Phi).
\end{align}
Next we provide a technical result dealing with the stability of our chain rule, which will prove to be useful in Section~\ref{sec:approx}.

\begin{lemma}\label{lem:stab}
It holds for almost every $x\in\R^d$ that
\begin{align}
    \lim_{y\to\r\Phi(x)}\big[(\nder\Psi)(y)-(\nder\Psi)(\r{\Phi}(x))\big]\cdot (\nder\Phi)(x)=0.
\end{align}
\end{lemma}
\begin{proof}
    We first show for every locally Lipschitz continuous function $u\colon\R^m\to\R^N$ and for almost every $x\in\R^d$ that
    \small
    \begin{align}\label{eq:induction}
        \lim_{y\to\r\Phi(x)}[H(u(y))-H(u(\r\Phi(x)))]\cdot(\der(u\circ\r\Phi))(x)=0.
    \end{align}
    \normalsize
    If $u_i(\r\Phi(x))\neq0$ we have
    \begin{equation}
        \lim_{y\to\r\Phi(x)}\1_{(0,\infty)}(u_i(y))=\1_{(0,\infty)}(u_i(\r\Phi(x)))
    \end{equation}
    as $u_i$ is continuous and $\1_{(0,\infty)}$ is continuous on $\R\backslash\{0\}$. Furthermore, \cite[Thm 3.3(i)]{Evans2015MeasureEdition} implies that 
    \begin{equation}
        (\der (u_i\circ\r\Phi))(x)=0
    \end{equation}
    for almost every $x\in\R^d$ with $u_i(\r\Phi(x))=0$. Since a finite union of nullsets is again a nullset, this proves the claim \eqref{eq:induction}. The lemma follows by induction over the layers $K=1,\dots,\tilde{L}$ of $\Psi$ and applying \eqref{eq:induction} with $u=\mathcal{R}_K\Psi$.
\end{proof}

\section{General Activation Functions}
As mentioned in the introduction, it is possible to replace the ReLU activation function in~\eqref{eq:Rdef} by some locally Lipschitz continuous, component-wise applied function $\varrho\colon\R\to\R$ with an at most countably large set $S$ of points where $\varrho$ is not differentiable. 
Specifically, one can define the neural network derivative (with activation function $\varrho$) as in~\eqref{eq:NNderDef} with $\mathbbm{1}_{(0,\infty)}(x_i)$ in~\eqref{eq:Hdef} replaced by 
\begin{equation}
    (\bar{\der} \varrho)(x_i):=\begin{cases} 0, & x_i\in S
    \\ (\der \varrho)(x_i), & \text{else}\end{cases}.
\end{equation}
The chain rule can, again, be checked by direct computation and it is straightforward to adapt Theorem~\ref{thm:main} to this more general setting by considering the level sets \begin{equation}
   \{x\in\R^d\colon w_i(x)=s\},\quad s\in S. 
\end{equation}
If additionally $\bar{\der} \varrho$ is continuous on $\R\setminus S$, the proof of Lemma~\ref{lem:stab} translates without any modifications. 

\section{Utilization in Approximation Theory} \label{sec:approx}
These results can now be employed to bound the $L^\infty$-norm of $\nder(\Psi\circ\Phi)-\der(u\circ \,v)$, given corresponding estimates for the approximation of $u$ and $v$ by $\Psi$ and $\Phi$, 
respectively. 
Here, one has to take some care when bounding the term
\begin{equation}
    \left\|\left[\nder \Psi \circ \r{\Phi}-\der u \circ \r{\Phi}\right]\nder \Phi\right\|_{L^\infty}
\end{equation}
by
\begin{equation}
   \|\nder \Psi -\der u \|_{L^\infty} \| \nder \Phi\|_{L^\infty}.
\end{equation}
Again it can happen that $\r{\Phi}$ maps a set of positive measure into a nullset where the estimate for the approximation of $\der u$ by $\nder \Psi$ in the \emph{essential} supremum norm is not valid. However, using the stability result in Lemma~\ref{lem:stab} one can for almost every $x\in\R^d$ shift to a sufficiently close point $y\approx \r{\Phi}(x)$ where the estimate holds. 
In \cite{Yarotsky2018} Yarotsky explicitly constructs networks whose realization is a linear interpolation\footnote{The interpolation points are uniformly distributed over the domain of approximation and their number grows exponentially with the size of the networks.} of the squaring function (see Fig.~\ref{fig:square} for illustration), which directly gives an estimate on the approximation rate for the derivatives.
These simple networks can then be combined to get networks approximating multiplication, polynomials and eventually, by means of e.g. local Taylor approximation, functions $f$ whose first $n\geq 1$ (weak) derivatives are bounded. This leads to estimates of the form
\begin{equation} \label{eq:est_B}
    \|f - \r\Phi_{\eps,B}\|_{L^\infty (I_B)} \le \eps,
\end{equation}
with $I_B=[-B,B]^d$,
including estimates for the scaling of the size of the network $\Phi_{\eps,B}$ w.r.t. $B$ and $\eps$.
As these constructions are based on composing simpler functions with known estimates one can now employ Theorem~\ref{thm:main} and Corollary~\ref{thm:chain_rule} to show that the derivatives of those networks also approximate the derivative of the function, i.e. 
\begin{equation}\label{eq:der_est_B}
     \|\der f - \nder \Phi_{\eps,B} \|_{L^\infty (I_B)} \le c \, \eps^{r}.
\end{equation} 
Such constructive approaches can further be found in~\cite{Elbrachter2018DNNPricing}, in \cite{Petersen2017} for $\beta$-cartoon-like functions, in \cite{SZ19_2592} for $(\bm{b},\eps)$-holomorphic maps, and in \cite{Grohs2019DeepTheory} for high-frequent sinusoidal functions. 
\begin{figure}
    \centering
    \includegraphics[width=0.9\linewidth]{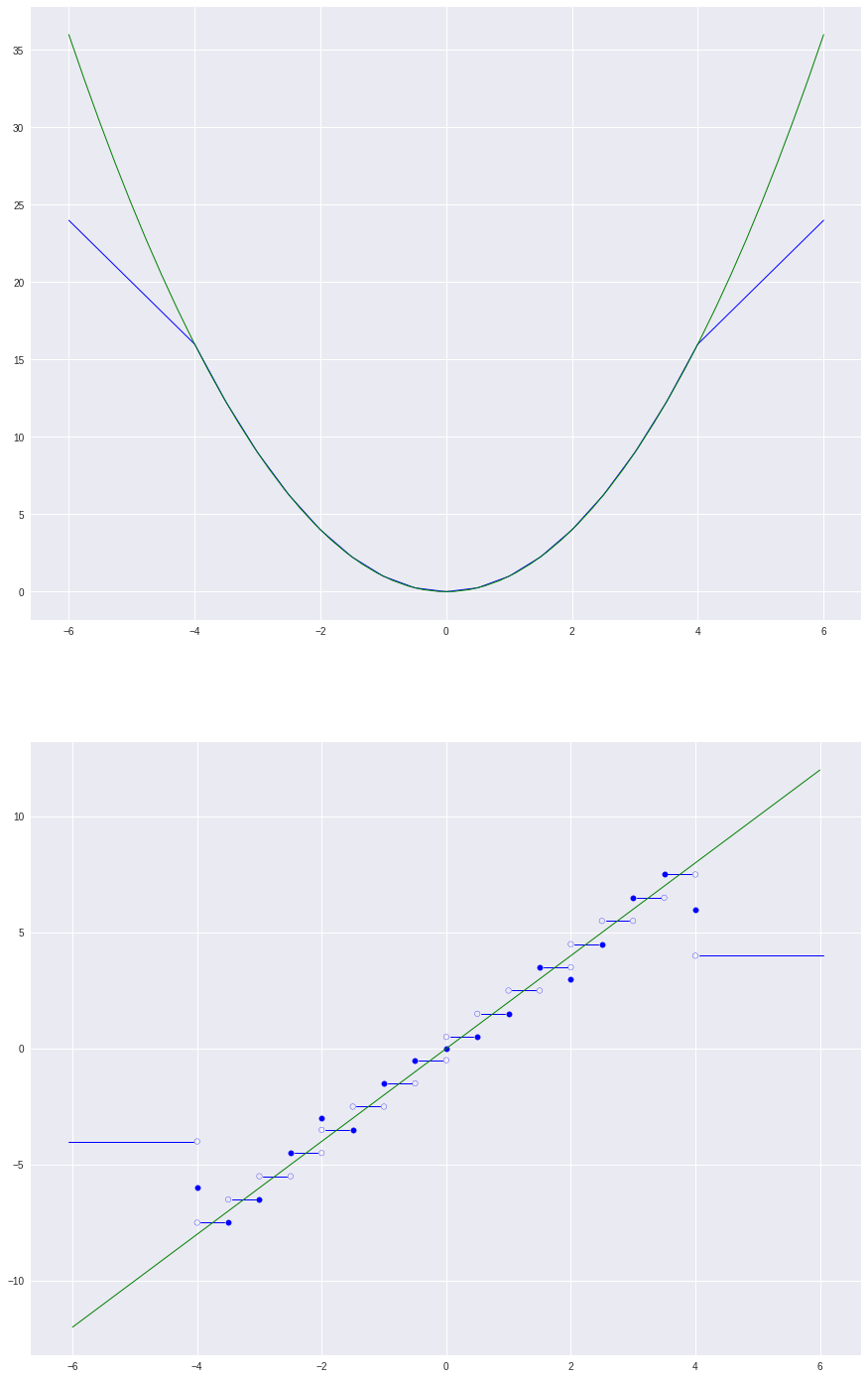}
    \caption{Approximation of the function $x\mapsto x^2$ and its derivative on the interval $[-4,4]$ by a neural network $\Phi$ with depth $6$, connectivity $52$ and weight bound $4$. Note that not all values of $\nder\Phi$ at the points of non-differentiablity of $\r\Phi$ lie between the values at either side, i.e. in the subdifferential.}
    \label{fig:square}
\end{figure} 

\section{Global Error Estimates}
The error estimates above are usually only sensible for bounded domains, as the realization of a neural network is always CPL with a finite number of pieces.
We briefly discuss a general way of transforming them into global pointwise error estimates, which can be useful in the context of PDEs (see e.g. \cite{grohs2018approx,Jentzen2018ACoefficients}).
In the following assume that we have a function $f$ with an at most polynomially growing derivative, i.e.
\begin{equation}
    \|(\der f)(x)\|_2 \le \bm{c} (1+\|x\|_2^{\kappa}).
\end{equation} 
Denote by $\Phi_B^{\operatorname{char}}$ a neural network which represents the $d$-dimensional approximate characteristic function of $I_B$, i.e. $\r{\Phi_B^{\operatorname{char}}}(x)\in [0,1]$ and 
\begin{align}\begin{split}
&\r{\Phi_B^{\operatorname{char}}}(x)=1,\quad  x\in I_B,\\ 
&\r{\Phi_B^{\operatorname{char}}}(x)=0,\quad
x\notin I_{B+1}.
\end{split}\end{align}
See \cite[Proof of Thm. VIII.3]{Grohs2019DeepTheory} for such a construction.
Further let $\Phi_{\eps,b}^{\operatorname{mult}}$ be the neural network approximating the multiplication function on $[-b,b]^2$ with error $\eps$ (see e.g.~\cite[Prop. 3.1]{SZ19_2592}).\\
Now we define the global approximation networks $\Phi_\eps$ as the composition of $\Phi_{\eps/2,b_\eps}^{\operatorname{mult}}$ with the parallelization of $\Phi_{B_\eps}^{\operatorname{char}}$ and $\Phi_{\eps/2,B_\eps+1}$ for suitable 
\begin{equation}
    B_\eps\in \mathcal{O}(\eps^{-1}) \quad \text{and} \quad b_\eps\in \mathcal{O}(\eps^{-\kappa-1}).
\end{equation}
See Figure~\ref{fig:proof_glob} for an illustration and e.g.~\cite[Def. 2.7]{Petersen2017} for a formal definition of parallelization.
Considering the errors on $I_B$, $I_{B+1}\backslash I_{B}$ and $\R^d\backslash I_{B+1}$ leads to global estimates, i.e. for every $x\in\R^d$
\begin{equation} 
|f(x)-\r{\Phi_\eps}(x)|\le \eps(1+\|x\|^{\kappa+2}_2) 
\end{equation}
and, by use of the chain rule~\ref{thm:chain_rule}, for almost every $x\in\R^d$
\begin{equation} \label{eq:glob_der}
    \|(\der f)(x) -(\nder\Phi_\eps)(x)\|_2 \le C\eps^r(1+\|x\|_2^{\kappa+2}).
\end{equation}
Due to the logarithmic size scaling of the multiplication network, the size of $\Phi_\eps$ can be bounded by the size of $\Phi_{\eps/2,B_\eps+1}$ plus an additional term in $\mathcal{O}(d+\kappa\log\eps^{-1})$.
\begin{figure}
\centering
\begin{tikzpicture}[node distance=1.5cm,>=stealth',on grid]
  \tikzstyle{place}=[rectangle,rounded corners,thick,draw=blue!75,fill=blue!20,minimum size=5mm,text centered]
  \tikzstyle{transition}=[rectangle,thick,draw=black!75,
  			  fill=black!20,minimum size=4mm]
  \tikzset{mystyle/.style={->}} 
    \node[place](x)  {\scriptsize $ x$};
    \coordinate[right of=x, node distance=3.3cm](center);
    \node[place,below of=center, node distance=1cm](phi) {\scriptsize $=\r{\Phi_{\eps/2,B_\eps+1}}(x)$};
    \node[place,above of=center, node distance=1cm](char) {\scriptsize $ \approx \mathbbm{1}_{[-B_\eps,B_\eps]^d}(x)$};
    \node[transition, right of=center, node distance=1.5cm](mult) {\scriptsize $  \Phi_{\eps/2,b_\eps}^{\operatorname{mult}}$};
    \node[place,right of=mult, node distance=2.2cm](final) {\scriptsize $ \approx \mathbbm{1}_{[-B_\eps,B_\eps]^d}f(x)$};
    \path (x) edge[mystyle] node[transition](par) {\scriptsize $ \Phi_{B_\eps}^{\operatorname{char}}$} (char.west)
          (x) edge[mystyle] node[transition](par) {\scriptsize $ \Phi_{\eps/2,B_\eps+1}$} (phi.west)
          (phi.east) edge (mult)
          (char.east) edge (mult)
          (mult) edge[mystyle] (final);
\end{tikzpicture}
\caption{The neural networks $\Phi_\eps$ approximating $f$ globally.}
\label{fig:proof_glob}
\end{figure}
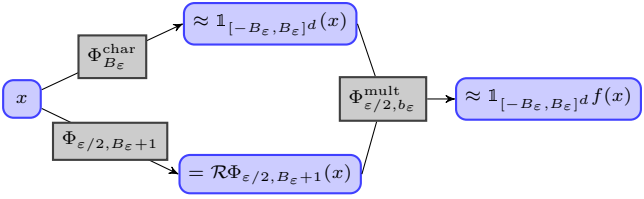
\section{Application to PDEs}
Analyzing the regularity properties of neural networks was motivated by the recent successful application of deep learning methods to PDEs~\cite{E2017,2017Jentzen,SirignanoSpiliopoulos2017,FujiiTakahashiTakahashi2017,KhooLuYing2017,EYu2017,Hutzen2019}. Initiated by empirical experiments~\cite{beckbecker2018} it has been proven that neural networks are capable of overcoming the curse of dimensionality for solving so-called Kolmogorov PDEs~\cite{Berner2018AnalysisEquations}. More precisely, the solution to the empirical risk minimization problem over a class of neural networks approximates the solution of the PDE up to error $\eps$ with high probability and with size of the networks and number of samples scaling only polynomially in the dimension $d$ and $\eps^{-1}$.
The above requires a suitable learning problem and a sufficiently good approximation of the solution function by neural networks. For Kolmogorov PDEs, this boils down to calculating global Lipschitz coefficients and error estimates for neural networks approximating the initial condition and coefficient functions (see e.g.~\cite{grohs2018approx,Jentzen2018ACoefficients}). 
Employing estimates of the form~\eqref{eq:der_est_B} one can bound the derivative on $I_B$, i.e. 
\begin{equation} \label{eq:lip_const}
\begin{split}
    L_B:=\|\nder \Phi_{\eps,B}\|_{L^\infty(I_B)} 
    \le \|\der f\|_{L^\infty(I_B)} +c\eps^r.
\end{split}
\end{equation}
Using mollification and the mean value theorem we can establish local Lipschitz estimates, i.e. for all $x,y\in (-B,B)^d$ that
\begin{equation}\label{eq:lip_bound}
        |\r{\Phi_{\eps,B}}(x)-\r{\Phi_{\eps,B}}(y)| \le L_B \|x-y\|_2,
\end{equation}
and corresponding linear growth bounds 
\begin{equation}
        |\r{\Phi_{\eps,B}}(x)| \le \big(|\r{\Phi_{\eps,B}}(0)|+L_B \big)(1+ \|x\|_2).
\end{equation}
Similarly, one can use~\eqref{eq:glob_der} to obtain
estimates of the form
\begin{equation}
   |\r{\Phi_\eps}(x)-\r{\Phi_\eps}(y)|\le \bm{C}(1+\|x\|^{\kappa+2}_2+\|y\|^{\kappa+2}_2)  \|x-y\|_2
\end{equation}
for all $x,y\in\R^d$ (which are demanded in \cite[Theorem 1.1]{Jentzen2018ACoefficients}). 
Moreover, note that the capability to produce approximation results which include error estimates for the derivative is of significant independent interest. 
Various numerical methods (for instance Galerkin methods) rely on bounding the error in some Sobolev norm $\|\cdot\|_{W^{1,p}}$, which requires estimates of the derivative differences. 
We believe that the possibility to obtain regularity estimates significantly contributes to the mathematical theory of neural networks and allows for further advances in the numerical approximation of high dimensional partial differential equations. 
\section{Relation to backpropagation in training}
The approach discussed here could further be applied to the training of neural networks by (stochastic) gradient descent.
Note, however, that this is a slightly different setting. From the approximation theory perspective we were interested in the derivative of $x\mapsto \r\Phi(x)$, while in training one requires the derivative of $\Phi\mapsto\r\Phi(x^*)$ for some fixed sample $x^*$. In particular this function is no longer CPL but rather continuous piecewise polynomial. While this would necessitate some technical modifications, we believe that it should be possible to employ the method used here in order to show that the gradient of $\Phi\mapsto\r\Phi(x^*)$ coincides \ae with what is computed by backpropagation using the convention of setting the derivative of $\max\{0,\cdot\}$ to zero at the origin (as well as similar conventions for e.g. max-pooling). 

\section*{Acknowledgment}
The research of JB and DE was supported by the Austrian Science Fund (FWF) under
grants I3403-N32 and P 30148.

\bibliographystyle{IEEEtran}
\bibliography{IEEEabrv,bib}
\end{document}